\documentclass[letterpaper, 10 pt, conference]{ieeeconf}  
\IEEEoverridecommandlockouts                              

\usepackage{xcolor}
\usepackage{amssymb}
\usepackage{graphicx}
\usepackage{mathtools}
\usepackage{amsmath}
\usepackage{gensymb}
\usepackage{float}
\usepackage{multirow}
\usepackage{makecell}
\usepackage[utf8]{inputenc}
\usepackage[english]{babel}

\usepackage{amsthm}

\usepackage{subcaption}

\newcommand{\taninv}{\tan^{-1}}
\newcommand{\no}{\noindent}

\newcommand{\mc}[1]{\mathcal{#1}}
\newcommand{\bb}[1]{\mathbb{#1}}

\newtheorem{prop}{Proposition}[section]
\DeclarePairedDelimiter\abs{\lvert}{\rvert}%

\newcommand\Tstrut{\rule{0pt}{2.6ex}}         

\title{\LARGE \bf
Frozone: Freezing-Free, Pedestrian-Friendly Navigation in Human Crowds
}

\author{Adarsh Jagan Sathyamoorthy, Utsav Patel, Tianrui Guan, and Dinesh Manocha
}

\begin{document}

\maketitle
\thispagestyle{empty}
\pagestyle{empty}

\begin{abstract}
We present Frozone, a novel algorithm to deal with the  Freezing Robot Problem (FRP) that arises when a robot navigates through dense scenarios and crowds. Our method senses and explicitly predicts the trajectories of pedestrians and constructs a Potential Freezing Zone (PFZ); a spatial zone where the robot could freeze or be obtrusive to humans. Our formulation computes a deviation velocity to avoid the PFZ, which also accounts for social constraints. Furthermore,  Frozone is designed for robots equipped with sensors with a limited sensing range and field of view. We ensure that the robot's deviation is bounded, thus avoiding sudden angular motion which could lead to the loss of perception data of the surrounding obstacles. We have combined Frozone with a Deep Reinforcement Learning-based (DRL) collision avoidance method and use our hybrid approach to handle crowds of varying densities. Our overall approach results in smooth and collision-free navigation in dense environments. We have evaluated our method's performance in simulation and on real differential drive robots in challenging indoor scenarios. We highlight the benefits of our approach over prior methods in terms of success rates (upto 50 \% increase), pedestrian-friendliness (100 \% increase) and the rate of freezing ( $> 80 \%$ decrease) in challenging scenarios.

\end{abstract}

\section{Introduction}
Mobile robots are increasingly used in many indoor scenarios. This includes applications such as waiters in hotels, as helpers in hospitals, as transporters of goods in warehouses, for surveillance, package delivery etc. To accomplish such tasks, these robots need to navigate through dense and challenging dynamic environments, specifically in crowds with pedestrian densities ranging from $<$ 1 to 2 persons/$m^2$. Apart from avoiding collisions with static and dynamic obstacles in its surroundings, the robot should also navigate in a pedestrian-friendly way. The latter includes satisfying social constraints \cite{fundamental-diag,sociosense,socially-aware}, such as maintaining sufficient distance from the pedestrians and other rules corresponding to avoiding them from behind (See Fig.1 Bottom).

A challenging problem that a robot could face in such scenarios is the Freezing Robot Problem (FRP) \cite{freezing1,freezing2}. FRP occurs when the robot faces a situation where the collision avoidance module declares that all possible velocities may lead to collisions. The robot either halts or starts oscillating indefinitely, which could either result in a collision or it does not make progress towards its goal. In practice, it is non-trivial to completely avoid FRP in crowds beyond a certain density without human cooperation \cite{freezing2}. One of the goals is to develop approaches that can reduce the occurrence of FRP in moderately dense crowds ($\le 1$ person/$m^2$) without assuming human cooperation, and using limited sensing capabilities.

There have been a few works addressing FRP \cite{freezing1,freezing2,rudenko2017Predictions,densecavoid}. Some approaches have also attempted to solve the problems of freezing and loss of localization in a crowd simultaneously \cite{unfrozen_Tingxiang}. While these methods are promising, we need more general solutions which rely less on global information and can provide some guarantees on the resulting performance. 

In order to address freezing, it is important to compute a collision-free trajectory for the robot based on local knowledge of the environment from its sensor data. The problen of collision avoidance is well studied and some of the widely used methods are based on velocity obstacles \cite{RVO,ORCA,NH-ORCA}, the dynamic window approach~\cite{DWA}, etc. More recently, many techniques have been proposed based on Deep Reinforcement Learning (DRL) collision avoidance \cite{Alahi,JHow1,JHow2,JiaPan1,WB1}. Decentralized DRL-based methods have gained popularity due to their superior performance in terms of success rates, average robot velocity, etc compared to traditional methods, and their ability to handle large numbers of dynamic obstacles, and can be robust to sensor uncertainty. However, it is hard to provide any guarantees.

\begin{figure}[t]
      \centering
      \includegraphics[width = 3.2in, height=2.9in]{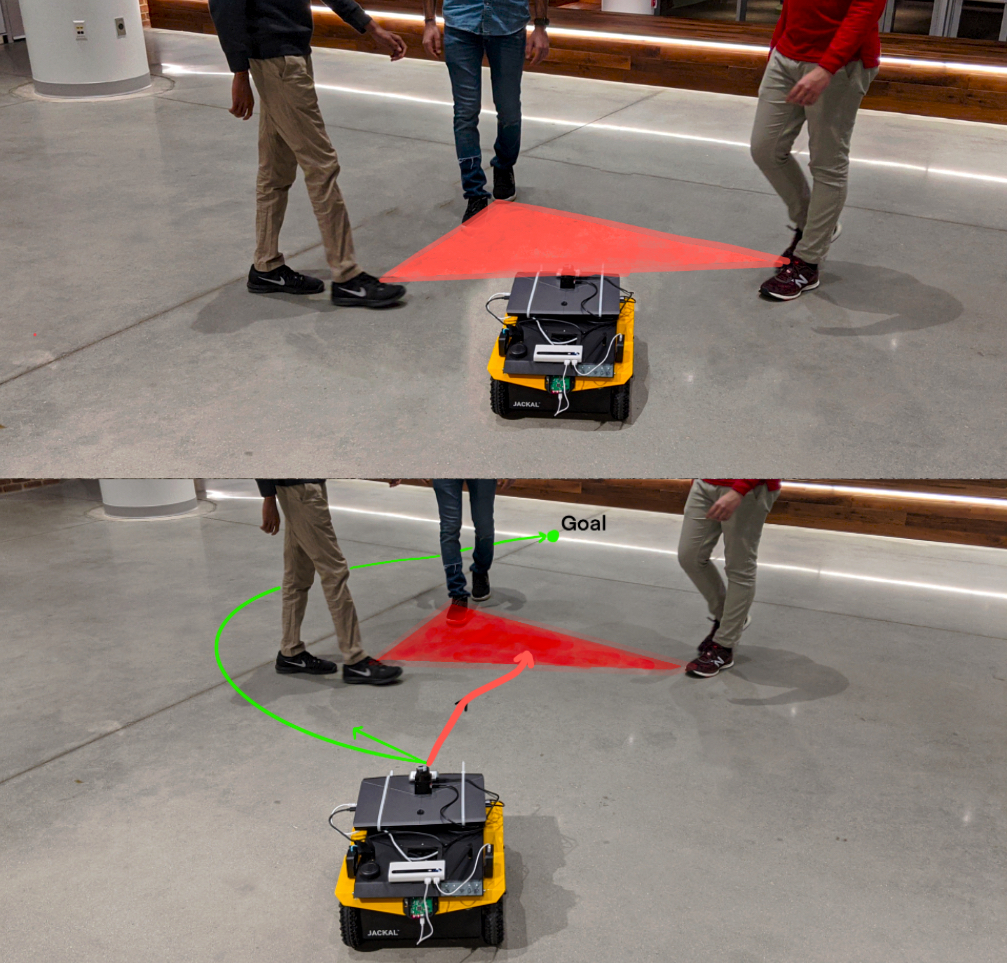}
      \caption {\textbf{[Top]} An instance of the Freezing Robot Problem (FRP), where the robot halts or oscillates indefinitely when it faces scenarios with pedestrians, as shown (highlighted in red). \textbf{[Bottom]} Another scenario where our approach implemented on a Clearpath Jackal robot navigates amongst pedestrians and preemptively avoids FRP. Our method explicitly tracks and predicts pedestrians' positions in the sensing range of the robot and classifies each pedestrian as \textit{potentially-freezing} or \textit{non-freezing}. We compute a \textit{potential} freezing zone (PFZ) (shown in red) and our method deviates the robot away from the freezing zone (green trajectory), while previous methods move the robot towards that zone (red trajectory).}
      \label{Cover}
      \vspace{-15pt}
\end{figure}



\textbf{Main Contributions:}

\begin{itemize}
    \item We present Frozone, a real-time algorithm that significantly reduces the occurrence of FRP by explicitly predicting  pedestrian trajectories, classifying them as \textit{potentially-freezing} or \textit{non-freezing} pedestrians and constructing a Potential Freezing Zone (PFZ). PFZ corresponds to a conservative spatial zone where the robot might freeze and be obtrusive to humans. Our method calculates a deviation to modify the robot's velocity  to avoid the PFZ leading to a decrease of more than 80\% in freezing rates over prior algorithms.
    
    \item Our method ensures that the robot's velocity avoids the PFZ, and is unobtrusive to the nearby pedestrians based on their social and psychological constraints for personal space. We observe an improvement in the pedestrian-friendliness of the robot's trajectories by 100\%. 
    
    
    \item We combine Frozone with a state-of-the-art DRL-based collision avoidance method and present a hybrid navigation algorithm that combines the benefits of traditional model-based algorithms such as better guarantees, and DRL-based approaches like better robustness to sensor uncertainty.
\end{itemize}

We evaluate our method in simulation and on a Clearpath Jackal (Fig. \ref{Cover}) and a Turtlebot (Fig. \ref{fig:turtlebot}) in several challenging indoor scenarios with crowds of varying densities ($<$ 1 person/$m^2$ to $> 2$ persons/$m^2$). 

\begin{figure}[t]
      \centering
      \includegraphics[width=\columnwidth]{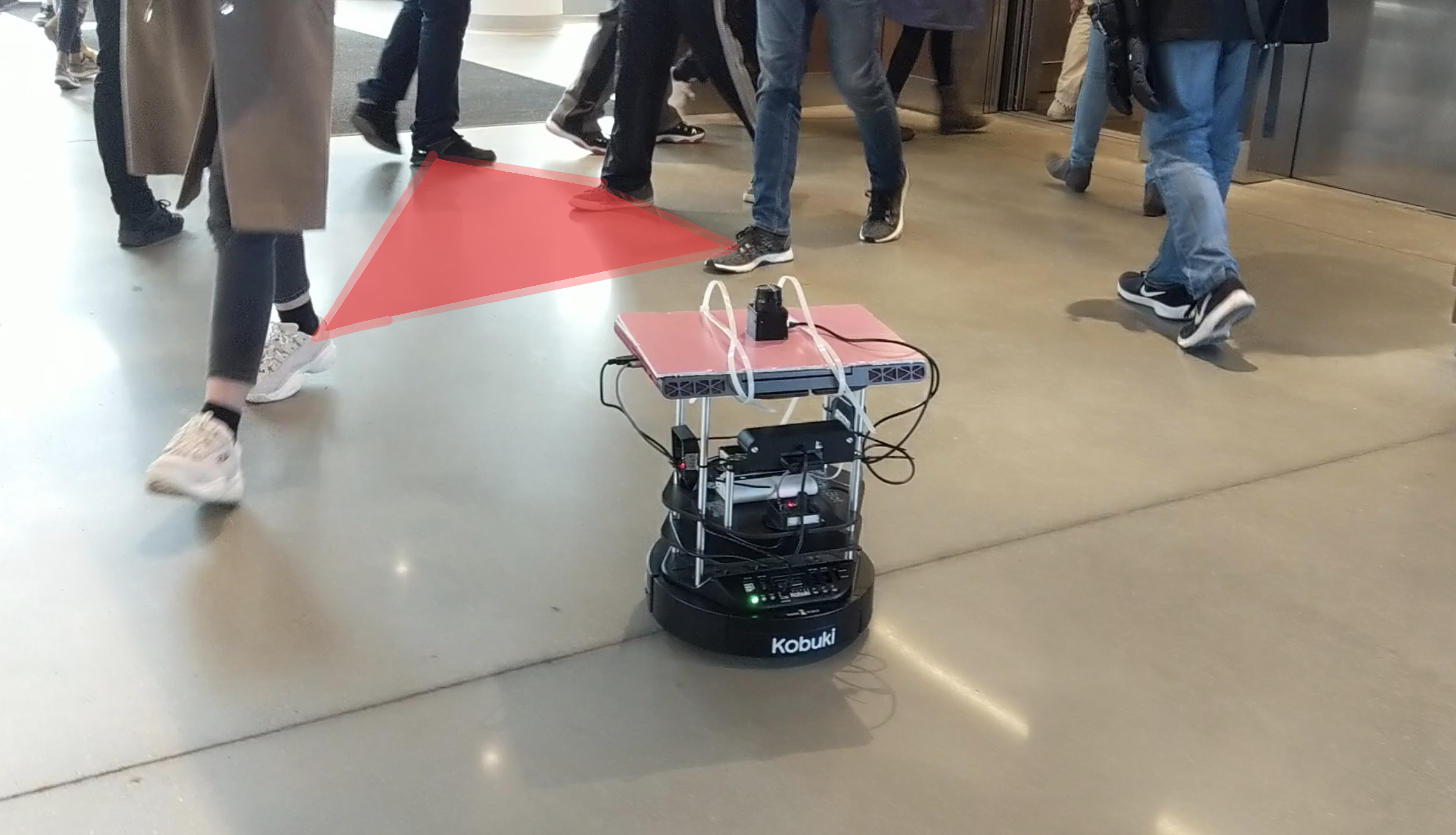}
      \caption {Frozone tested on a Turtlebot in dense crowds. Our tests demonstrate the applicability of our method on different mobile robots with limited sensing capabilities equipped with different perception sensors. The Potential Freezing Zone is represented in red.}
      \label{fig:turtlebot}
      \vspace{-15pt}
\end{figure}

\section{Related Work}
We discuss the relevant previous work on the Freezing Robot Problem, DRL-based collision avoidance methods, and socially-aware navigation.

\subsection{Freezing Robot Problem}
The earliest works addressing FRP \cite{freezing1,freezing2} argue that it can only be solved by accounting for human-robot cooperation~\cite{case4cooperation} in terms of robot navigation, meaning that the robot and the humans need to adjust their trajectories. These methods also show that even perfect pedestrian trajectory prediction would not help in solving FRP, unless human-robot cooperation is explicitly modeled. Other attempts at solving FRP include a method based on learning from demonstration \cite{YuWang_Learch} and improved motion prediction based on the Markov Decision Process \cite{rudenko2017Predictions}. The resulting algorithms use static CCTV cameras or simulations to validate their methods and do not fully account for many practical issues that arise on a real robot such as the loss of sensing data due to the robot's motion. Reducing FRP using a trained deep learning policy using implicit pedestrian prediction \cite{densecavoid} has also been investigated. Our approach is complimentary to these methods and is more robust in terms of reducing FRP.

\subsection{Deep Reinforcement Learning-based Collision Avoidance}
In recent years, methods based on DRL have been used for navigation in dense scenes. Some methods have had great success with training a decentralized collision avoidance policy \cite{JiaPan1} and combining it with traditional control strategies \cite{JiaPan2}. The trained DRL policy has been extended to solve the loss of localization and FRP simultaneously by learning recovery points \cite{unfrozen_Tingxiang}. While these DRL methods model cooperative behaviors between humans and robots implicitly, other methods explicitly model cooperation using a value network with two agents \cite{JHow1} or an arbitrary number of agents \cite{JHow2}. Intra-crowd interactions, which could indirectly affect a robot's navigation, have also been studied \cite{Alahi}. DRL-based methods have also been extended to navigate in a  socially acceptable manner \cite{socially-aware}, providing safety guarantees by identifying previously unseen scenarios and performing more cautious maneuvers \cite{JHow-uncertainty}. Our hybrid algorithm can also be combined with these DRL-based approaches.

\subsection{Socially Aware Navigation}
Robots navigating among pedestrians should reduce the amount of discomfort to the surrounding humans. Prior work in crowd or pedestrian simulation provides insights on the psychological and environmental factors that affect the pedestrian's motion or trajectory~\cite{fundamental-diag,densesense}, which should be accounted during robot navigation. Works on socially aware robot navigation include predicting the long-term trajectories of pedestrians using personality trait theory and Bayesian learning~\cite{sociosense} and classifying group emotions~\cite{entitivity}. Learning-based methods can be used to account for social norms by specifying the behaviors that should be avoided \cite{socially-aware-Jon-How}. In our work, we use these ideas and focus on computing robot velocities that maintain a comfortable distance in-front of pedestrians.

\section{Background and Overview}
We define the notation used in our approach and provide some background on the different components used. We also discuss the factors that affect pedestrian motion in a crowd and use them to model the social constraints.

\subsection{Notation and Symbols}
We represent each pedestrian as $[\textbf{p}^{ped}, \textbf{u}^{ped}] = [p^{ped}_{x}, p^{ped}_{y}, u^{ped}_{x}, u^{ped}_{y}] \in \mathbb{R}^4$, where $\textbf{p}^{ped}$ and $\textbf{u}^{ped}$ represent the 2-D position and unit vector representing the front/forward direction of the pedestrian, respectively. We assume that the robot knows its relative goal location and denote it as $\textbf{g}^{rob}$. All values are specified \textit{with respect to a coordinate frame attached to the robot,} with an origin denoted by $\textbf{o}^{rob}$ as shown in Fig. \ref{freezing-scenario}(a). Therefore, the forward heading direction of the robot always corresponds to (1 $\hat{i} + 0\hat{j}$) and the direction towards the left is represented as (0 $\hat{i} + 1\hat{j}$). We represent scalar values in normal fonts and vectors using bold fonts.

We assume that the robot is equipped with a depth camera to sense nearby pedestrians. Its sensing region $\mathcal{S}^{rob}$ is formulated as a square space with a side length of \textit{$s_{sen}$} meters in front of the robot (see Fig.\ref{freezing-scenario}(a)). Our approach can be easily modified for other sensing regions. The sensing region is offset in front of the robot by \textit{f} meters to account for sensing inaccuracies in the depth image, which could arise for objects too close to the robot. The depth image \textit{I} at any time instant \textit{t}, and the value of a pixel at $(i, j)$ which contains the proximity of an object at that part of the image, are represented as:
\begin{equation}
\begin{split}
    I^t &= \{C \in \mathbb{R}^{w \times h} : f < C_{ij} < f + s_{sen}\}, \\
    1 &\le i \le w \qquad \text{and} \qquad 1 \le j \le h,  
\end{split}
    \label{cam}
\end{equation}
\no where \textit{w}, \textit{h}, \textit{i} and \textit{j} are the image's width, height, and the indices along the width and height, respectively. We define \textit{dist(\textbf{a}, \textbf{b})} as a function that computes the Euclidean distance between points \textbf{a} and \textbf{b}.

\subsection{Pedestrian Behavior in crowds}
Pedestrian motions in crowds are influenced by several factors such as crowd density, individual stride length, and need for personal space. Pedestrians tend to walk slower when there is less space in front of them (i.e., dense crowd). The fundamental diagram~\cite{fundamental-diag} is used to model this behavior through an inverse relationship between pedestrian velocities and the crowd density.

A pedestrian's natural walking speed ($v^{ped}$) is related to physiological (pedestrian's height and stride length) and psychological (need for personal space) factors \cite{fundamental-diag} using the following equation: 
\begin{equation}
\resizebox{15em}{!}{
    $v^{ped} = min\left(||\Vec{v}^{pref}||, \left(\frac{S\alpha}{H(1 + \beta)}\right)^2\right)$}
    \label{ped-walking-speed}
\end{equation}
where $||\Vec{v}^{pref}||$ is the pedestrian's preferred speed directed towards its goal, which we assume to be 1.3 meters/second (on average). S is the available space in front of the pedestrian, H (height/1.72) is a height normalization factor, and $\alpha$ and $\beta$ are constants that account for the stride of the pedestrian. We assume that all the pedestrians have the same values for H, $\alpha$ and $\beta$. We use the velocity relationship in Equation \ref{ped-walking-speed} to predict the future positions of pedestrians after a time $\Delta t$. This formulation is complimentary to existing trajectory prediction methods, and also accounts for human psychological need for space while walking. From Equation \ref{ped-walking-speed}, we infer that a pedestrian's motion is unaffected when a robot avoids/passes them from behind or if it leaves sufficient space in front of them (high value of S). 

Therefore, in our approach, we compute a deviation for the robot velocity such that the robot maintains sufficient distance in-front of the pedestrians, and avoids them from behind when possible. We define $\eta$ as the pedestrians' comfortable distance threshold based on social constraints, i.e. the distance that the robot needs to maintain from a pedestrian whenever possible. In addition, we define the \textit{pedestrian-friendliness} \textit{(PF)} of the robot as, 

\begin{equation}
    PF = Z*N_{\infty} + \Bar{Z}*min(dist(\textbf{o}^{rob}, \textbf{p}^{ped}_i)),    
    \label{eq:PF}
\end{equation}
\no where Z is a boolean that represents if a pedestrian has been avoided from behind, $N_{\infty}$ is a large value, and $\textbf{p}^{ped}_i$ represents the position of the $i^{th}$ pedestrian in the environment. 

\subsection{Pedestrian Detection and Tracking}
Object detection and tracking are used to locate the pedestrians in the image from the depth camera and mark their trajectories accurately. These problems are well studied in computer vision and robotics literature, and recently, good solutions are proposed based on deep learning. We use a pre-trained YOLOv3 \cite{redmon2018yolov3} model for pedestrian detection in the depth images. YOLOv3 outputs a set of bounding boxes for the detected pedestrians as $\mc{B}= \{ \bb{B}_{k} \ | \ \bb{B} = [\textrm{top left}, m_{\bb{B}}, n_{\bb{B}}],  \in \mc{H} \}$, where $\textrm{top left}, m_{\bb{B}},$ and $n_{\bb{B}}$ denote the top left corner coordinates, width, and height of the $k^{th}$ bounding box $\bb{B}_k$, respectively. $\mc{H}$ denotes the set of all pedestrian detections. In addition, it also assigns an ID for a given image and reports the detection confidence. 

The detection bounding boxes are used as inputs to DensePeds~\cite{chandra2019densepeds}, a state-of-the-art pedestrian tracking algorithm that assigns a unique ID to each detected pedestrian across multiple consecutive images. This ID is used to compute each pedestrian's position in the image over time. We use DensePeds as it is robust to the noise in the image and performs with an accuracy of over 93\% in dense ($> 1$ person/$m^2$) scenarios.

\subsection{DRL-based Collision Avoidance}
As mentioned previously, we interface an end-to-end DRL-based collision avoidance policy \cite{JiaPan1} with Frozone. The policy uses observations from a 2-D lidar ($\textbf{o}^t_{lidar}$), the robot's relative goal location ($\textbf{o}^t_{goal}$), and its current velocity ($\textbf{o}^t_{vel}$) to compute new collision avoiding velocities at each time instant. During training, a reward/penalty function is shaped to: (i) minimize the robot's time to reach its goal, (ii) reduce oscillatory motions in the robot, (iii) head towards the robot's goal, and most importantly, (iv) avoids collisions. Post training, a collision free velocity $\textbf{v}^{DRL}$ at each time instant is sampled from a trained policy $\pi_{\theta}$ as: 
\begin{equation}
    \textbf{v}^{DRL} \sim \pi_{\theta}(\textbf{\textbf{v}}^t | \textbf{o}^t),
\end{equation}
\no where $\textbf{v}^t$ and $\textbf{o}^t$ are the velocity and the observation spaces at time t, respectively.


\section{Our Method: Frozone}
In this section, we describe how we formulate the freezing robot problem. This includes various components including pedestrian classification, the construction of a potential freezing zone, and our formulation to calculate the deviation needed to avoid the freezing zone.

\begin{figure*}[t]
      \centering
      \includegraphics[width = \textwidth, height = 2.6in]{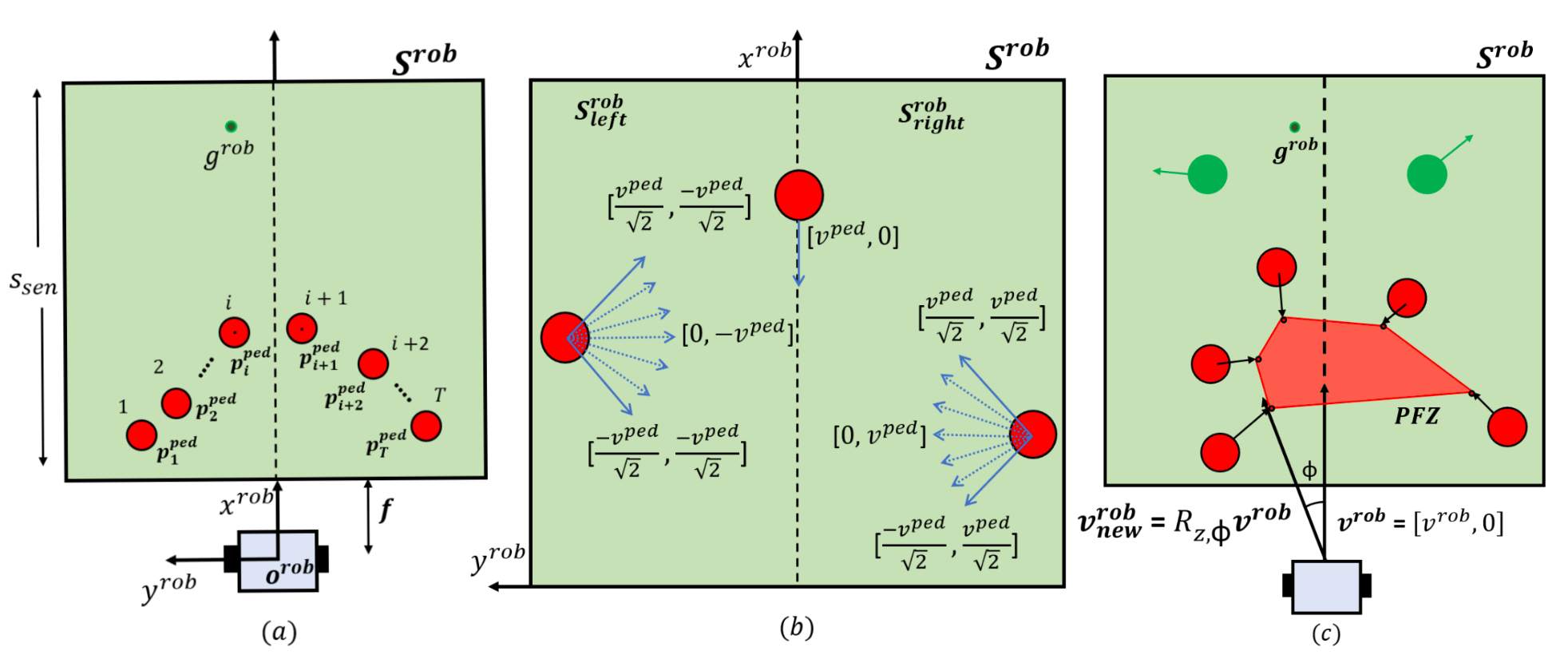}
      \caption {\textbf{(a)} An example for the freezing robot scenario. The robot does not find enough space between pedestrians (red circles) to move towards its goal and the robot's navigation module concludes that all velocities towards the goal would lead to a collision. \textbf{(b)} The range of pedestrian velocity directions that are considered as \textit{potentially-freezing} in the left and right half planes of the sensing region. Pedestrian velocities along the robot's X-axis are special cases of an either head-on approach or a pedestrian moving away. \textbf{(c)} The construction of the Potential Freezing Zone (PFZ), which is the convex hull of the predicted positions of \textit{potentially-freezing} pedestrians after time $\Delta t$. The deviation angle $\phi$ is computed such that the robot is directed away from the PFZ with the least amount of deviation from its current velocity.}
      \label{freezing-scenario}
      \vspace{-15pt}
\end{figure*}

\subsection{Formulation of FRP}
Consider a scenario as shown in Fig. \ref{freezing-scenario}(a), where a robot faces a set of \textit{T} pedestrians in its sensing region before reaching its goal. The pedestrians could either be mobile or stationary.

\textbf{Definition- IV.1} \textit{(Freezing Robot Problem)}. If the pedestrians are stationary and are positioned such that

\begin{equation}
\begin{split}
    \textit{dist}(\textbf{p}^{ped}_i, \textbf{p}^{ped}_{i + 1}) &< 2\Omega \quad \forall i \in \{1, 2, ..., \textit{T}\}, \\
    dist(\textbf{o}^{rob}, \textbf{p}^{ped}_i) &\le \Omega \quad \forall i \in \{1, 2, ..., \textit{T}\},\\
    \textbf{p}^{ped}_i, \textbf{p}^{ped}_{i + 1} &\in \mathcal{S}^{rob},
\end{split}
\label{eqn:FRP}
\end{equation}

\no where $\textbf{p}^{ped}_i$ and $\textbf{p}^{ped}_{i+1}$ are the positions of the $i^{th}$ and $(i+1)^{th}$ pedestrian with respect to the robot, and $\Omega$ denotes the minimum distance threshold that the robot's collision avoidance module should maintain with all obstacles. If the conditions in \ref{eqn:FRP} are satisfied then the planner deems all forward velocities as unsafe. Such a scenario constitutes the \textit{Freezing Robot Problem}. One possible technique (without human cooperation) for the robot to reach its goal is to retrace its path, identify the free space, and re-plan its trajectory. This typically requires global knowledge of the environment, such as a map of the environment and all dynamic obstacles, which may not readily available. The robot either halts completely or generates undesirable behaviors such as oscillations, which severely degrade the efficiency of the robot's navigation.

If the humans are non-stationary, the robot disrupts their velocities due to its low proximity in front of them (i.e. low value of S). This reduces the robot navigation's pedestrian-friendliness. Therefore, we develop an approach where the robot is able to predict if such scenarios could occur in the near future and preemptively avoid them. This simultaneously reduces the occurrence of FRP and improves the robot's pedestrian-friendliness in terms of social constraints. 

\subsection{Computing Pedestrian Poses and Prediction}
The pedestrian positions and orientations are obtained by tracking pedestrians in the image produced by the depth camera. Referring to Equation \ref{cam}, the value of the $(i, j)^{th}$ pixel in the image $C_{ij}$ contains the proximity of an obstacle present in that part of the image. When a pedestrian \textit{k} is detected and tracked in consecutive images (Section III.C), the pixel values within the detection bounding box $\bb{B}_k$ are averaged to measure the mean distance ($d_{avg}$) of the pedestrian from the camera. Let the centroid of the bounding box $\bb{B}_k$ be denoted as $[x_{\bb{B}_k}, y_{\bb{B}_k}]$. Then, the angular displacement $\psi$ of the pedestrian with respect to the robot can be calculated as:
\begin{equation}
    \psi = \left(\frac{x_{\bb{B}_k}}{w}\right) * FOV_{cam},
\end{equation}
\no where $FOV_{cam}$ is the field of view angle of the camera. The pedestrian's position with respect to the robot can be calculated as $[p^{ped}_x, p^{ped}_y]$ = $d_{avg}$ * [$\cos{\psi}, \sin{\psi}$]. A pedestrian's orientation unit vectors $[u^{ped}_x, u^{ped}_y]$ can be computed by calculating the difference between the pedestrian positions computed between two consecutive depth images. The pedestrian's perceived space in front of the pedestrian (S) can be trivially calculated based on the computed relative positions of all the pedestrians in the robot's sensing region. Using S, the pedestrian's walking speed $v^{ped}$ is calculated based on Equation \ref{ped-walking-speed}.  

\subsection{Classifying Potentially Freezing Pedestrians}
To predict if the freezing scenario discussed in Section IV.A could occur, our navigation algorithm first identifies the pedestrians in its sensing region who could cause such freezing behavior and classify them as \textit{potentially-freezing}. To this end, the sensing region $\mathcal{S}^{rob}$ of the robot is divided equally along the robot's X-axis as $\mathcal{S}^{rob}_{right}$ and $\mathcal{S}^{rob}_{left}$ as shown in Fig. \ref{freezing-scenario}(b). Consider a pedestrian in the sensing region of the robot with an orientation denoted by \textbf{u} = [$u^{ped}_{x}, u^{ped}_{y}$]. The pedestrian's velocity vector $\textbf{v}^{ped}$ with respect to the robot's frame, can be obtained by scaling \textbf{u} by the pedestrian walking speed $v^{ped}$ (from Equation \ref{ped-walking-speed}) as 
\[
\textbf{v}^{ped} = [v^{ped}_{x}, v^{ped}_{y}] = v^{ped} * [u^{ped}_{x}, u^{ped}_{y}].
\]Let the robot speed be $v^{rob}$. Therefore, its velocity vector with respect to its local coordinate frame will be $\textbf{v}^{rob} = v^{rob}*[1, 0]$. A pedestrian positioned at $\textbf{p}^{ped} \in \mathcal{S}^{rob}_{right}$ is considered as \textit{potentially-freezing} if its velocity vector $\textbf{v}^{ped}$ satisfies these constraints:
\begin{equation}
\begin{split}
    v^{ped}_{x} &\in [-v^{ped}/\sqrt{2}, v^{ped}/\sqrt{2}], \\
    v^{ped}_{y} &\in [v^{ped}/\sqrt{2}, v^{ped}]. \\
\end{split}
     \label{right-relevant-vec}
\end{equation} 
Similarly, the velocity vectors $\textbf{v}^{ped}$ of a pedestrian positioned at $\textbf{p}^{ped} \in \mathcal{S}^{rob}_{left}$ is considered \textit{potentially-freezing} if it satisfies: 
\begin{equation}
\begin{split}
    v^{ped}_{x} &\in [-v^{ped}/\sqrt{2}, v^{ped}/\sqrt{2}],\\
    v^{ped}_{y} &\in [-v^{ped}, -v^{ped}/\sqrt{2}].\\
\end{split}
    \label{left-relevant-vec}
\end{equation}
The pedestrian's speed is assumed to be comparable to the robot's speed in both cases ($v^{ped} \sim v^{rob}$).
\begin{prop}
The distance function (dist()) between any pedestrian with a velocity vector that satisfies the conditions in equations \ref{right-relevant-vec} or \ref{left-relevant-vec} and the robot, is a decreasing function with time.
\end{prop}

\begin{proof}
Consider $\textbf{p}^{ped} \in \mathcal{S}^{rob}_{right}$ with a velocity vector $[v^{ped}_{x}, v^{ped}_{y}]$ that satisfies Equation \ref{right-relevant-vec}. From Fig.\ref{freezing-scenario}(b), we observe that $p^{ped}_x$ is positive and $p^{ped}_y$ is negative. For simplicity, let us assume that for a time interval [$t_0, t_{fin}$], $v^{ped} = v^{rob}$. Let $t_{fin} - t_0 = \Delta t$ be a small time interval comparable to the time taken for the depth camera to capture two consecutive images. In this time interval, with respect to the robot, the pedestrian would have moved from $(p^{ped}_x, p^{ped}_y)$ to

\begin{equation}
    (p^{ped}_x + v^{ped}_{x}\Delta t - v^{rob}\Delta t, p^{ped}_y + v^{ped}_{y}\Delta t).
\end{equation}
\no Since $\Delta t$ is small, there cannot be any sudden changes in the pedestrian's motion. Since $v^{ped} = v^{rob}$, $v^{ped}_{x} \in [-v^{ped}/\sqrt{2}, v^{ped}/\sqrt{2}] \implies  v^{ped}_x < v^{rob}$ and $v^{ped}_{y} \in [v^{ped}/\sqrt{2}, v^{ped}] \implies v^{ped}_{y} > 0$. Then, the distance between the robot and the pedestrian at $t_0$ is given as,

\begin{equation}
    d_0 = \sqrt{(p^{ped}_x)^2 + (p^{ped}_y)^2}.
\end{equation}

\no The distance between the robot and the pedestrian at $t_f$ is given as,
\begin{equation}
    d_f = \sqrt{(p^{ped}_x + v^{ped}_x\Delta t - v^{rob}\Delta t)^2 + (p^{ped}_y + v^{ped}_y\Delta t)^2}.
\end{equation}

\no Since $v^{ped}_x\Delta t - v^{rob}\Delta t < 0$ and $\abs{p^{ped}_y + v^{ped}_y\Delta t} < \abs{p^{ped}_y}$, we get $d_f < d_0 \implies dist(\textbf{o}^{rob}, \textbf{p}^{ped})$ is a decreasing function in [$t_0, t_{fin}$]. A similar proof can be formulated for a pedestrian in $\mathcal{S}^{rob}_{left}$. This result implies that the pedestrian velocities satisfying conditions \ref{right-relevant-vec} and \ref{left-relevant-vec} move closer the robot, potentially causing freezing. Their motion would also be affected the most by the robot's navigation. Therefore, such pedestrians should be classified as \textit{potentially-freezing}. \\
\end{proof} 

\noindent {\bf Pedestrian Classification:} 
The \textit{potentially-freezing} velocity directions for each half of the sensing region are shown in Fig. \ref{freezing-scenario}(b). If $v^{ped} < v^{rob}$, then the pedestrian is considered as \textit{potentially-freezing}, irrespective of his/her orientation, as the distance between the robot and the pedestrian decreases over time. Pedestrians with velocity vectors of the form $[\pm v^{ped}, 0]$, are cases where the pedestrian is either moving head-on towards the robot or moving away from the robot (maybe at a lower speed than the robot). Such pedestrians are also considered potentially-freezing provided that the pedestrian position is along the X-axis of the robot. All other pedestrians are considered as \textit{non-freezing}. All \textit{potentially-freezing} pedestrians' positions are involved in \textit{freezing} zone computation.  

\subsection{Constructing the Potential Freezing Zone}
\textbf{Definition IV.2} \textit{(Potential Freezing Zone (PFZ)).} PFZ is defined as a conservative region with a high probability for the occurrence of FRP, after a time interval $\Delta t$. Using the current position and velocity of a \textit{potentially-freezing} pedestrian \textit{i}, we predict his/her's position $\hat{\textbf{p}}^{ped}_i$ after $\Delta t$ as, 

\begin{equation}
    \hat{\textbf{p}}^{ped}_i = \textbf{p}^{ped}_i + \textbf{v}^{ped}_i\Delta t \quad i \in {1,2,...,K,}
\end{equation}
\no where K is the number of \textit{potentially-freezing} pedestrians in the sensing zone. With these predicted positions as vertices, a closed region is constructed, as in (Fig. \ref{freezing-scenario}(c)),
\begin{equation}
    PFZ = Convex Hull (\hat{\textbf{p}}^{ped}_i), \quad i \in {1,2,...,K.}
\end{equation}


\no We use the convex hull of $\hat{\textbf{p}}^{ped}_i$ as a conservative approximation because it simplifies the computation of the freezing zone. It also simplifies the deviation angle computation explained in Section IV.E. The cases where K = 1 is a special case where the PFZ is a single point. To maintain sufficient distance from the pedestrian, we construct the PFZ as a circle centered around the pedestrian with a fixed radius. The avoidance of the PFZ locally guarantees the prevention of FRP after time $\Delta t$. In addition, it improves the robot's pedestrian-friendliness with respect to the K pedestrians, as the robot's deviation away from the PFZ ensures that it does not navigate obtrusively in-front of the pedestrians.

\subsection{Calculating Deviation Angle $\phi$}
At any instant, if $\hat{\textbf{p}}^{ped}_i$ of the closest \textit{potentially-freezing} pedestrian \textit{i} positioned at $\textbf{p}^{ped} = [p^{ped}_x, p^{ped}_y]$ satisfies, 

\begin{equation}
dist(\textbf{v}^{rob}\Delta t, \hat{\textbf{p}}^{ped}_i) \le \eta    
\end{equation}

\no ($\eta$ is the pedestrian comfort distance) and $\textbf{v}^{rob}\Delta t \in PFZ$, the robot initializes a deviation from its current velocity direction to a new velocity by an angle $\phi$ as, 
\begin{equation}
\textbf{v}^{rob}_{new} = R_{z, \phi}\textbf{v}^{rob},     
\end{equation}

\begin{figure}[t]
      \centering
      \includegraphics[width = 3.0in]{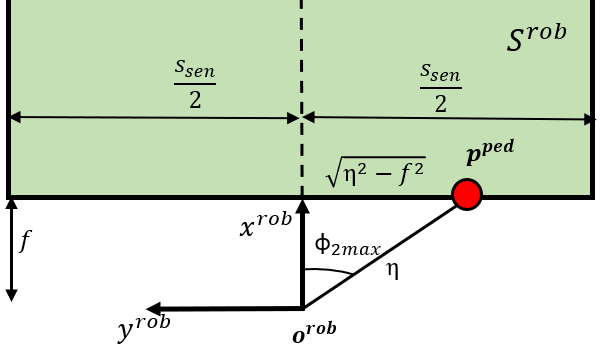}
      \caption {The maximum possible deviation $\phi_{2max}$ occurs, when the closest \textit{potentially-freezing} pedestrian is at $[f, s_{sen}/2]$. This proves that $\phi_2$ is always bounded, which implies that $\phi = min(\phi_1, \phi_2)$ is bounded.}
      \label{fig:max-deviation}
      \vspace{-15pt}
\end{figure}

\no where $R_{z, \phi}$ denotes the rotation matrix for an angle $\phi$ about the robot's Z-axis. This deviation is constrained based on the relative location of the robot's goal and avoids navigating in-front of the pedestrians. $\phi$ can be computed as,
\begin{equation}
    \phi = min(\phi_1, \phi_2),
\end{equation}
\no where $\phi_1$ and $\phi_2$ are given by,
\begin{gather}
    \phi_1 = \underset{R_{z,\phi_1}\textbf{v}^{rob}\Delta t \notin PFZ}{\operatorname{argmin}} \left(dist(R_{z,\phi_1}\textbf{v}^{rob}\Delta t, \textbf{g}^{rob})\right), \\
    \phi_2 = \taninv(p^{ped}_y / p^{ped}_x), \quad \phi_2 \ne 0.
\end{gather}

\no If $min(\phi_1, \phi_2) = \phi_2$, the robot deviates towards the closest pedestrian's current position. Since the pedestrian has a non-zero velocity in the interval $\Delta t$, the robot avoids the closest pedestrian from behind. If $\phi_2 = 0$, we use $\phi = \phi_1$, which denotes the least deviation from the goal. 
\begin{prop}
The deviation angle $\phi$ has an upper bound of $\taninv(\pm \sqrt{\eta^2 - f^2} / f)$ which depends on $\eta$ and f.  
\end{prop}
\begin{proof}
Consider the formulation for $\phi_2$. For the dimensions of the sensing region shown in Fig. \ref{freezing-scenario}(a), the maximum value of the deviation angle for a certain $\eta$ could occur if the closest pedestrian is located at $\textbf{p}^{ped}$ with ${p}_x^{ped}$ = \textit{f} meter (just within the sensing region), and ${p}_y^{ped} = \pm \sqrt{\eta^2 - f^2}$ (see Fig. \ref{fig:max-deviation}). This is because the $\taninv(p^{ped}_y / {p}^{ped}_x)$ function increases as $p^{ped}_x$ decreases, and \textit{f} is the least possible value that $p_x^{ped}$ can have. Then, the maximum value of $\phi_2$, $\phi_{2max} = \taninv(\pm \sqrt{\eta^2 - f^2} / f)$. $\phi_2$ is bounded by $\taninv(\pm \sqrt{\eta^2 - f^2} / f) \implies \phi$ is bounded by $\taninv(\pm \sqrt{\eta^2 - f^2} / f)$.
\end{proof}

\no The constraints on the deviation angle ensure that (i) PFZ can be avoided which reduces freezing and improves pedestrian-friendliness; (ii) since $\phi$ is the least possible deviation away from the PFZ, the angular motion of the robot is restricted. This results in minimizing the loss of line of sight of the obstacles that are surrounding the robot. We note that our deviation angle computation does not assume pedestrian cooperation for collision avoidance. This formulation is applicable in moderate to dense crowds ($\le$ 1 person/ $m^2$),  where human cooperation is not required to avoid freezing. 


\subsection{Frozone and Deep Reinforcement Learning}
We use a hybrid combination of a DRL-based collision avoidance policy by Long et al \cite{JiaPan1} and Frozone to compute the robot's collision-free velocities. The overall system architecture is shown in Fig. \ref{system-arch}. Frozone modifies the velocities computed by the DRL policy to preemptively avoid potential freezing zones in crowds with lower pedestrian densities ($\le$ 1 person/$m^2$). In dense crowds, the DRL-based policy demonstrates good collision avoidance capability, and the computed velocities are used directly by our hybrid method. 

The crowd density in the environment is categorized based on the total number of pedestrians in the robot's sensing region T. Our switching mechanism for different densities can be expressed as:
\begin{equation}
    \textbf{v}^{rob} = 
    \begin{cases}
        \textbf{v}^{Frozone} \quad \text{if} \,\, T \le s_{sen}^2, \\
        \textbf{v}^{DRL} \qquad \text{if} \, \, T > s_{sen}^2.
    \end{cases}
\end{equation}
\no $T \le s_{sen}^2$ corresponds to a density of $\le 1 person/m^2$. We also note that Frozone is not restricted to be interfaced with a DRL-based method and can be employed with any collision avoidance scheme.




\begin{figure}[t]
      \centering
      \includegraphics[width = \columnwidth, height=1.8in]{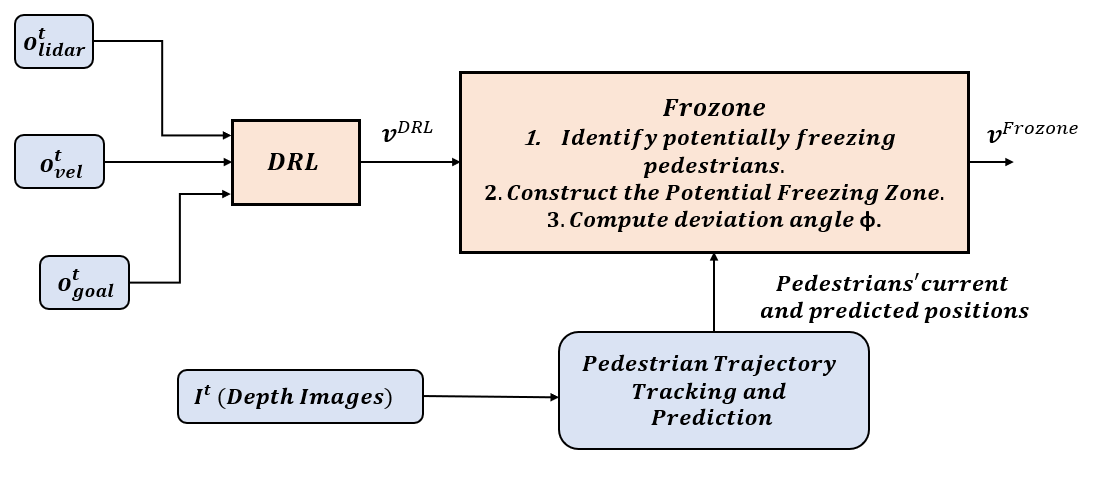}
      \caption {Our hybrid method's system architecture, which uses the velocity provided by a DRL method \cite{JiaPan1} as a guiding velocity for collision avoidance and modifies it to preemptively avoid Potential Freezing Zones (PFZ). Our method includes 3 main components: 1. Classifying pedestrians as potentially-freezing and non-freezing; 2. Constructing a \textit{Potential} Freezing Zone; 3. Avoiding the PFZ in a pedestrian-friendly manner.}
      \label{system-arch}
      \vspace{-15pt}
\end{figure}

\section{Results and Analysis}
In this section, we describe our implementation and highlight its performance in different evaluation scenarios. We also compare our method with prior methods and show significant improvements in terms of navigation performance. 

\subsection{Implementation}
We first evaluate our method in simulations that were created using ROS Kinetic and Gazebo 8.6. We use a simulated model of a Clearpath Jackal robot attached with models of the Hokuyo 2-D lidar and the Orbbec Astra depth camera in Gazebo. The Hokuyo lidar, used by the DRL-based method, has a range of $4$ meters, an FOV of 240$^\circ$, and provides $512$ range values per scan. The Astra camera has a minimum and maximum sensing range of $0.5$ meters and $5$ meters, respectively. We use images of size $ w \times h = 150 \times 120$ with added Gaussian noise $\mathcal{N}(0,0.2)$ as inputs to our pedestrian detection and tracking system. For our real-world implementation, we mount use different cameras; the Astra camera with a Turtlebot 2 and an Intel Realsense depth camera with a Clearpath Jackal robot .


\subsection{Testing Scenarios}
We recreate typical indoor and outdoor scenarios that the robot could face in our simulator to evaluate the performance of our algorithm and to compare with prior approaches. For sparse crowds, we assume that the robot takes full responsibility to avoid the collisions. As a result, we make no assumptions about human cooperation. We consider the following challenging scenarios to evaluate the performance:

\no \textbf{Corridor:} The robot must navigate through 15 pedestrians in a straight corridor to reach its goal. The scenario is made more challenging by making the pedestrians walk in multiple pairs or in a zig-zag manner.

\no \textbf{Crossing:} The robot must move perpendicular to the pedestrian motion in a plus-shaped corridor. The pedestrians may not be sensed until they are very close to the robot. 

\no \textbf{Random-5:} The robot must navigate in a random manner and move through five pedestrians that result in high pedestrian density in a local region. The maximum pedestrian density in this scenario is 1 person / $m^2$.

\no \textbf{Random-10:} The robot navigates through randomly moving 10 pedestrians. The maximum density however, is $<$ 0.75 person / $m^2$. This scenario is used to evaluate the maximum number of pedestrians that each method can handle at any time instant. 

\no \textbf{1 Pedestrian Head-on:} To compare the reduction in the occurrence of freezing and increase in pedestrian-friendliness, we make the robot approach a single pedestrian head-on from different initial positions that are at a distance of $3$ meters and $4$ meters. The pedestrian moves at $1$ m/s towards the robot and halts in-front of the robot, emulating a freezing scenario in dense crowds. This tests the collision avoidance response time of each algorithm, and if the algorithm can avoid freezing. 

\no \textbf{1 pedestrian Perpendicular:} The robot moves perpendicular to a pedestrian's motion. We evaluate if the robot avoids the pedestrian from front (obtrusive) or back (unobtrusive). 

We use the above mentioned single pedestrian scenarios for our evaluating freezing rates and pedestrian-friendliness, as it provides a more precise way to measure these parameters and observe robot behaviors when compared to more dense scenes.

\subsection{Evaluation Metrics}
We highlight the various metrics used to evaluate the algorithm. The mean time and velocity are self-explanatory and correspond to the values when the robot reached its goal position without a collision. 

\textbf{Success Rate} - The number of times that the robot reached its goal without collisions over the total number of attempts.

\textbf{Freezing Rate} - The number of times the robot got stuck or started oscillating for more than 10 seconds, while avoiding obstacles over the total number of attempts.

\textbf{Pedestrian friendliness} - We use Equation \ref{eq:PF} to compute the pedestrian friendliness metric. We use $N_{\infty} = 10$ in our results. 

\subsection{Analysis}
Table \ref{Tab:Results} shows the results of our comparisons between Frozone + Long et al.~\cite{JiaPan1} DRL method (hybrid) and three previous methods in our simulator: (i) Dynamic Window Approach \cite{DWA}, a traditional collision avoidance method which uses the lidar to sense nearby obstacles and forward-simulate the robot's motion to detect potential collisions, (ii) Long et al.~\cite{JiaPan1}, a DRL-based collision avoidance method for dense crowd collision avoidance; (iii) DenseCAvoid \cite{densecavoid}, a DRL-based method with a trained policy that reduces FRP using pedestrian tracking and prediction. As mentioned before, our implementation of Frozone is combined with Long et al's DRL policy \cite{JiaPan1} for our evaluations.

We observe that Frozone + Long et al's method has the best success rates of all the methods in all the scenarios highlighted above. It significantly improves the success rates over just using Long et al.'s~\cite{JiaPan1} algorithm, and performs better or comparably with DWA. This is because a robot's success rate is closely tied to avoiding freezing as well. Our hybrid method improves the mean time to goal when compared to Long et al.'s method by up to 33\%, and the robot's average velocity increases up to 40 \%, while resulting in comparable performances with the other methods.

\textbf{Freezing and Pedestrian-Friendliness:} Frozone significantly reduces the freezing in all our challenging test scenarios. In the scenario where a pedestrian approaches the robot head-on starting from 3 meters away, all previous methods halted to avoid a collision and started oscillating for more than 15 seconds. Frozone + Long et al's method avoids the circular PFZ constructed around the pedestrian preemptively and prevents freezing. When the pedestrian starts from 4 meters away, DWA and DenseCAvoid manage to not freeze in some cases, since the robot has more space and time to react to the pedestrian, and avoid it. Both these methods however, do not handle dynamic pedestrians in close proximity ($<$ 2 meters away).

Frozone also deviates the robot to avoid pedestrians from behind whenever it results in least deviation from the goal. When the robot moves perpendicular to the pedestrian's motion, all previous methods take short-sighted actions and try to pass the pedestrian from the front. In some cases, DWA deviated more than 5 meters away from the goal to avoid the pedestrian. Such behavior severely affects both the pedestrian-friendliness and the efficiency of the navigation. In all our trials, Frozone + Long et al's method passed the pedestrian from behind, without changing the perceived space in-front of the pedestrian. 

\begin{figure}[t]
      \centering
      \includegraphics[width=\columnwidth]{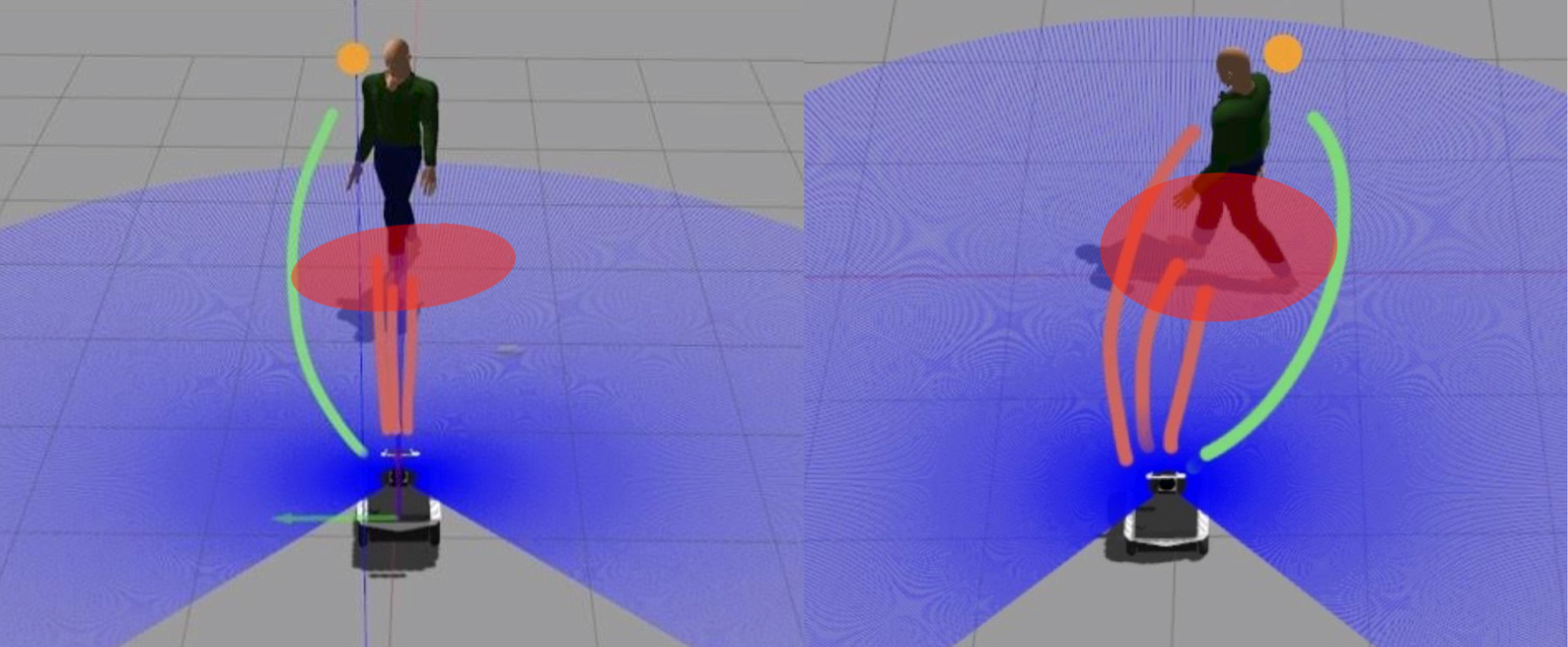}
      \caption {The trajectories of our Frozone + Long et al's method hybrid (in green) when compared with the trajectories of previous methods (red trajectories). Our method always avoids the PFZ (red circular region constructed around the pedestrian) and thus avoids freezing. Our method also avoids the pedestrian from behind, thereby not affecting the pedestrian's motion. The yellow point represents the robot's goal.}
      \label{fig:results}
      \vspace{-15pt}
\end{figure}

In our real-world tests on two differential drive robots, we observed that the robot could switch between the velocities computed by Long et al's method, and Frozone seamlessly. The performance in real-world scenarios is also simplified as humans generally cooperate with avoiding collisions with the robots. We also test the pedestrian-friendliness improvement in real-world scenes and observe that the robot conservatively navigates behind a pedestrian whenever possible (please see submitted video). 

\begin{table}[!htb]
\caption{ \small{\label{Tab:Results} Comparison of the hybrid combination of Frozone + Long et al's method's performance with other learning methods \cite{JiaPan1,densecavoid}, and a traditional collision avoidance method DWA \cite{DWA} in challenging scenarios. We observe that our hybrid combination's performance is better or comparable to previous methods. We represent our hybrid combination as Frozone + DRL in the table.}}
 \resizebox{\columnwidth}{!}{%
\begin{tabular}{|c|c|c|c|c|c|} 
\hline
\textbf{Metrics}\Tstrut  & \textbf{Method} & \textbf{Corridor} & \textbf{Crossing} & \textbf{Random-5} & \textbf{Random-10}\\ [0.5ex] 
\hline
\multirow{4}{*}{\rotatebox[origin=c]{0}{\makecell{\textbf{Success}\\\textbf{Rate}\\(higher\\better)}}} & DWA \cite{DWA}\Tstrut & 0.4 & 0.6 & 0.6 & 0.8 \\
 & Long et al. \cite{JiaPan1} & 0.1 & 0.4 & 0.6 & 0.2 \\
 & DenseCAvoid \cite{densecavoid} & 0.2 & 0.6 & 0.7 & 0.4 \\
 & Frozone + DRL & 0.6 & 0.8 & 0.6 & 0.8 \\
\hline
\multirow{4}{*}{\rotatebox[origin=c]{0}{\makecell{\textbf{Mean}\\\textbf{Time}\\(lower\\better)}}} & DWA \cite{DWA}\Tstrut & 32.9 & 31.9 & 33.25 & 32.71 \\
 & Long et al. \cite{JiaPan1} & 35.4 & 39.2 & 31.3 & 33.8 \\
 & DenseCAvoid \cite{densecavoid} & 32.03 & 39.08 & 32.77 & 30.04 \\
 & Frozone +DRL  & 31.3 & 29.06 & 32.8 & 32.2 \\
\hline
\multirow{4}{*}{\rotatebox[origin=c]{0}{\makecell{\textbf{Avg}\\\textbf{Velocity}\\(higher\\better)}}} & DWA \cite{DWA}\Tstrut & 0.30 & 0.31 & 0.31 & 0.30 \\
 & Long et al. \cite{JiaPan1}  & 0.28 & 0.25 & 0.30 & 0.33 \\
 & DenseCAvoid \cite{densecavoid}
 & 0.33 & 0.25 & 0.31 & 0.33 \\
 & Frozone + DRL & 0.34 & 0.35 & 0.33 & 0.33 \\
\hline
\end{tabular}
}
\end{table}

\begin{table}[!htb]
\caption{ \small{\label{Tab:Results2} Comparison of different methods in terms of freezing rates and pedestrian-friendliness (PF) for challenging 1-person scenarios, when the pedestrian starts at close proximity from the robot and moves at 1 m/s. Our Frozone + Long et al's method hybrid (represented as Frozone + DRL) outperforms all previous methods, with a significant decrease in freezing and increase in pedestrian-friendliness in all scenarios.}}
\resizebox{\columnwidth}{!}{
\begin{tabular}{|c|c|c|c|c|c|} 
\hline
\textbf{Metrics}\Tstrut \Tstrut & \textbf{Method} & \textbf{1Ped-3m} & \textbf{1Ped-4m} & \textbf{Ped-perp-3m} & \textbf{Ped-perp-4m} \\ [0.5ex] 
\hline
\multirow{4}{*}{\rotatebox[origin=c]{0}{\makecell{\textbf{Freezing}\\\textbf{Rate}\\(lower\\better)}}} & DWA \cite{DWA} \Tstrut & 100 & 87 & 11 &  0\\
 & Long et al. \cite{JiaPan1} & 100 & 100 & 23 & 11\\
 & DenseCAvoid \cite{densecavoid} & 100 & 81 & 5 & 0\\
 & Frozone + DRL & 0 & 0 & 0 & 0\\
\hline

\multirow{6}{*}{\rotatebox[origin=c]{0}{\makecell{\textbf{PF}\\(higher\\better)}}}  &  &  &  &  &  \\
& DWA \cite{DWA}  & 0.0 & 0.35 & 1.02 & 1.5 \\
 & Long et al. \cite{JiaPan1}  & 0.0 & 0.0 & 0.57 & 0.95\\
 & DenseCAvoid \cite{densecavoid} & 0.0 & 0.24 & 0.56 & 1.32\\
 & Frozone + DRL & 0.36 & 0.52 & 10 & 10\\
  &  &  &  &  &  \\
\hline
\end{tabular}
}
\end{table}

\section{Conclusions, Limitations and Future Work}
We present  a novel method to significantly reduce the occurrence of the Freezing Robot Problem when a robot  navigates through moderately dense crowds. Our approach is general and uses a standard camera for pedestrian detection and tracking. We present a simple algorithm to construct a spatial zone, where the robot could freeze and be obtrusive to the pedestrians. We use this formulation of potentially freezing zone to compute an angle to deviate the robot away from the zone. We also combined our method with a DRL-based collision avoidance method to exploit its advantages in dense crowds ($>$1-2 persons/$meter^2$), while reducing the occurrence of freezing. We observe that our method outperforms existing collision avoidance methods by having lower freezing rates, higher pedestrian-friendliness and success rates of reaching the goal. We highlight the performance on two different robots.

Our approach has certain limitations. While our method can reduce the rate of freezing, but we cannot avoid it altogether without human cooperation. Our formulation of potentially freezing zone is conservative and we use a locally optimal technique to compute the deviation. The behavior of our approach is also governed by the underlying pedestrian tracking algorithm as well as the techniques used to model pedestrian friendliness. Our hybrid method's performnce is also governed by the DRL formulation. As part of future work, we would like to overcome these limitations and test our approach with other robots. We would also like to take into account the dynamics constraints of the robot.


\bibliographystyle{IEEEtran}
\bibliography{References}
\end{document}